\documentclass[letterpaper, 10 pt, conference]{ieeeconf}

\IEEEoverridecommandlockouts

\font\stixfrak=stix-mathfrak at 10pt
-mathcal at 10pt

\usepackage{dutchcal}
\usepackage{amsmath}
\usepackage{amsfonts}
\usepackage{amssymb}
\usepackage{amsxtra}
\usepackage{amsthm} 
\usepackage{leftidx}
\usepackage{url}
\usepackage{graphicx}
\usepackage[dvipsnames]{xcolor}
\usepackage{float}
\usepackage{stfloats}
\usepackage{dsfont}
\usepackage{mathrsfs}
\usepackage{array}
\usepackage{rotating}
\usepackage{calc}
\usepackage{tikz-cd}
\usepackage{dsfont}
\usepackage{stmaryrd}
\usepackage{cancel}
\usepackage{accents}
\let\labelindent\relax
\usepackage{enumitem}
\usepackage[linesnumbered,lined,boxed,commentsnumbered]{algorithm2e}
\SetAlFnt{\footnotesize}
\SetAlCapFnt{\rmfamily\footnotesize}
\usepackage[colorlinks=true,urlcolor=black,linkcolor=blue,citecolor=blue]{hyperref}
\usepackage[capitalize,nameinlink]{cleveref}
\usepackage{comment}
\usepackage{soul}
\usepackage{marginnote}

\usepackage{mathrsfs}

\usepackage{multicol}
\usepackage{subcaption}

\makeatletter
\newcommand\Func[2]{%
    \textbf{function} #1
    \algocf@group{#2}%
}
\makeatother

\makeatletter
\newcommand\Forr[2]{%
    \textbf{for} #1 \textbf{do}%
    \algocf@group{#2}%
}
\makeatother

\makeatletter
\newcommand\Blnk[2]{%
    \hspace{10pt}#1
    \algocf@group{#2}%
    \textbf{end}
}
\makeatother

\usepackage{color} 

\usepackage[normalem]{ulem}

\makeatletter
\newcommand{\removelatexerror}{\let\@latex@error\@gobble}
\makeatother

\makeatletter
\def\@endtheorem{\endtrivlist}
\makeatother

\newtheorem{theorem}{Theorem}
\newtheorem{definition}{Definition}
\newtheorem{proposition}{Proposition}

\newtheorem{lemma}{Lemma}

\date{\today}

\newcommand{\nn}{{\mathscr{N}\negthickspace\negthickspace\negthinspace\mathscr{N}}\negthinspace}
\newcommand{\ou}{%
  \mathrel{%
    \vcenter{\offinterlineskip
      \ialign{##\cr$<$\cr\noalign{\kern-1.5pt}$>$\cr}%
    }%
  }%
}%

\newcommand{\bigtimes}{{\sf X}}

\renewcommand{\r}[1]{{\color{BrickRed}#1}}
 
\newcommand{\g}[1]{{{{\color{Green}#1}}}} 
\newcommand{\e}[1]{{{{\color{Orange}#1}}}} 

\renewcommand{\r}[1]{{\color{Black}#1}}
 
\renewcommand{\g}[1]{{{{\color{Black}#1}}}} 
\renewcommand{\e}[1]{{{{\color{Black}#1}}}} 

\renewcommand{\marginnote}[1]{}

\newcommand{\lblkbrbrack}{\negthinspace\text{{\stixfrak\char"36}}\normalfont}
\newcommand{\rblkbrbrack}{\text{{\stixfrak\char"37}}\normalfont}

\newcommand{\subarg}[1]{\lblkbrbrack #1 \rblkbrbrack}

\newcommand{\myalg}{L-TLLBox}

\renewcommand{\baselinestretch}{0.91}

\makeatletter
\newcommand{\removeonelatexerror}{\let\@latex@error\@gobble}
\makeatother

\begin{document}

\title{
\LARGE{\bf Polynomial-Time Reachability for LTI Systems with Two-Level Lattice Neural Network Controllers
}
} %
\author{James Ferlez\textsuperscript{$*$} and Yasser Shoukry\textsuperscript{$*$}
\thanks{
\textsuperscript{$*$}Department of Electrical Engineering and Computer Science, University of California, Irvine
\texttt{\{jferlez,yshoukry\}@uci.edu}
} %
\thanks{This work was supported in part by NSF awards \#CNS-2002405 and \#CNS-2013824 and the C3.ai Digital Transformation Institute.} %
}

\maketitle

\begin{abstract}
	In this paper, we consider the computational complexity of bounding the 
	reachable set of a Linear Time-Invariant (LTI) system controlled by a 
	Rectified Linear Unit (ReLU) Two-Level Lattice (TLL) Neural Network (NN) 
	controller. In particular, we show that for such a system and controller, 
	it is possible to compute the exact one-step reachable set in 
	\emph{polynomial} time in the size of the TLL NN controller 
	(number of neurons). Additionally, we show that a tight bounding box of the 
	reachable set is computable via two polynomial-time methods: one with 
	polynomial complexity in the size of the TLL and the other with polynomial 
	complexity in the Lipschitz constant of the controller and other problem 
	parameters. 
	Finally, we propose a pragmatic algorithm that adaptively combines the 
	benefits of (semi-)exact reachability and approximate reachability, which 
	we call \myalg. We evaluate \myalg~with an empirical comparison to a 
	state-of-the-art NN controller reachability tool. In our experiments, 
	\myalg~completed reachability analysis as much as 5000x faster than this 
	tool on the same network/system, while producing reach boxes that were from 
	0.08 to 1.42 times the area.
\end{abstract}


\section{Introduction} 
\label{sec:introduction}
Neural Networks (NNs) are increasingly used to control dynamical systems in 
safety critical contexts. 
As a result, the problem of \emph{formally} verifying the safety properties of 
NN controllers in closed loop is a crucial one.
Despite this, comparatively little attention has been paid to the 
\emph{time-complexity} of such reachability analysis. Understanding -- and 
improving -- the complexity of NN verification algorithms is thus crucial to 
designing provably safe NN controllers: it bears directly on the size of NNs 
that can be pragmatically verified.


Formal verification of NNs is usually formulated in terms of static 
input-output behavior, but there are few results analyzing the time complexity 
of such input-output verification \cite{KatzReluplexEfficientSMT2017a, 
FerlezBoundingComplexityFormally2020, TranNNVNeuralNetwork2020}. We know of no 
paper that directly analyzes the time-complexity of exact reachability analysis 
for LTI systems with NN controllers, although 
\cite{TranStarBasedReachabilityAnalysis2019} comes closest.  \marginnote{\g{(E0)}} %
\g{Note: exact reachability is distinct from (polynomial-time) set-based 
reachability methods, which consider an over-approximated set of possible 
controller outputs in each state \cite{AlthoffSetPropagationTechniques2021}.} 
Formally, \cite{TranStarBasedReachabilityAnalysis2019} only provides a 
complexity result for verifying the input-output behavior of a NN, but the 
underlying methodology, star sets, suggests a complexity analysis for exact 
reachability of LTI systems. Unfortunately, that algorithm produces 
exponentially many star sets -- in the number of neurons -- just to verify the 
input-output behavior of a NN once \cite[Theorem  
1]{TranStarBasedReachabilityAnalysis2019}; this exponential complexity 
compounds with each additional time step in reachability analysis. No such 
analysis is provided for the accompanying approximate star-set reachability 
analysis.


In this paper, we show that for a certain class of ReLU NN controllers -- 
viz. Two-Level Lattice (TLL) NNs \cite{FerlezAReNAssuredReLU2020} -- exact (or  
quantifiably approximate) reachability analysis for a controlled discrete-time 
LTI system is worst-case polynomial time complexity in the size (number of 
neurons) of the TLL NN controller. \g{Thus, we show that LTI reachability analysis\marginnote{\g{(E1)}} %
for the TLL NN architecture is dramatically more efficient (per neuron) than 
the same problem with general NNs  (i.e. exponential complexity 
\cite{TranStarBasedReachabilityAnalysis2019}; see above). In this sense, our 
results motivate for \emph{directly designing TLL NN controllers in the first 
place}, since reachability for a TLL NN controller is more efficient to compute 
(TLL NNs are similarly beneficial in other problems: e.g. verification 
\cite{FerlezFastBATLLNN2022}). Moreover, TLL NNs can realize \emph{the same 
functions} that general ReLU NNs can\footnote{See the TLL form of Continuous 
Piecewise-Affine functions 
\cite{TarelaRegionConfigurationsRealizability1999}.}, so no generality in 
realizable controllers is lost by this choice.}

In particular, we prove several polynomial-complexity results related to the 
\emph{one-step reachable set} of a discrete-time LTI system: i.e., the set 
$X_{t+1} = \{ A x + B \nn(x) | x \in X_{t} \}$ for a given polytopic set of 
states\footnote{\marginnote{\g{(E3)}}\g{Polytopic input constraints are a 
natural -- and ubiquitous -- choice, since ReLU NNs are affine on convex 
polytopic regions; hence, our complexity results are  also expressed in terms 
of the complexity of a Linear Program.}} $X_t$ and a TLL controller $\nn$. 
Moreover, we consider the computation of both the \emph{exact} set $X_{t+1}$ 
and an $\epsilon$-tight \emph{bounding box} of $X_{t+1}$. All claimed 
complexities are worst case and with respect to \emph{a fixed state-space 
dimension}\footnote{The reachability (verification) problem for a NN alone is 
known to be able to encode satisfiability of any 3-SAT formula; in particular, 
this result matches 3-SAT variables to input dimensions to the network 
\cite{KatzReluplexEfficientSMT2017a}.}, $n$. \marginnote{\e{(E2)}}\e{These 
results are summarized as:

\begin{enumerate}[left=0.1\parindent,label={\itshape (\roman*)}]
	\item The exact one-step reachable set, $X_{t+1}$, can be computed in 
		polynomial time-complexity in the size of the TLL NN 
		(\cref{thm:one_step_exact}).

	\item An $\epsilon$-tight bounding box for $X_{t+1}$ can be computed via 
		three algorithms with time-complexities:

		\begin{enumerate}[left=0.2\parindent,label={\itshape (\alph*)}]
			\item polynomial in size of the TLL 
				(\cref{thm:tll_size_one_step_box}); or

			\item polynomial in the Lipschitz constant of the controller, 
				the accuracy, $\epsilon$, the norm of the $B$ matrix and the 
				volume of $X_t$ (\cref{prop:generic_lipschitz_reachability}); or 

			\item the minimum complexity of \emph{(ii-a)} and \emph{(ii-b)} 
				(\cref{thm:box_alg_choice}); this uses polynomial-time 
				Lipschitz constant computation for TLLs 
				(\cref{lem:lipschitz_bound}).
		\end{enumerate}
\end{enumerate}
Here an $\epsilon$-tight bounding box of $X_{t+1}$ is one that is within 
$\epsilon > 0$ of the exact, coordinate-aligned bounding box of $X_{t+1}$.}



In addition, we propose an algorithm that \emph{adaptively} combines notions of 
exact and approximate \emph{bounding box} reachability for TLL NNs in order to 
obtain an extremely effective approximate reachability algorithm, which we call 
\myalg. 
We validate this method by empirically comparing an implementation of 
\myalg\footnote{\url{https://github.com/jferlez/FastBATLLNN}} with the 
state-of-art NN reachability tool, NNV \cite{TranNNVNeuralNetwork2020}. On a 
test suite of TLL NNs derived from the TLL Verification Benchmark in the 2022 
VNN Competition \cite{vnn2022}, \myalg~performed LTI reachability analysis  as 
much as 5000x faster than NNV on the same reachability problem; \myalg~produced 
reach boxes of 0.08 to 1.42 times the area produced by NNV.

\noindent \textbf{Related work:} 
There is a large literature on the complexity of set-based reachability for LTI 
systems; \cite{AlthoffSetPropagationTechniques2021} provides a good summary. 
For the complexity of LTI-NN reachability, 
\cite{TranStarBasedReachabilityAnalysis2019} is the closest to providing an  
explicit, exact result. The complexity of approaches based NN  
over-approximation have been considered in 
\cite{HuangPOLARPolynomialArithmetic2022, IvanovVerifyingSafetyAutonomous2020}. 
The literature on the complexity of input-output verification of NNs is larger 
but still small: \cite{TranStarBasedReachabilityAnalysis2019} falls in this 
category as well; \cite{KatzReluplexEfficientSMT2017a} is important for its 
NP-completeness result based on the 3-SAT encoding; and 
\cite{FerlezBoundingComplexityFormally2020, FerlezFastBATLLNN2022} consider the 
complexity of verifying TLL NNs. Other NN-related complexity results include: 
computing the minimum adversarial disturbance is NP hard 
\cite{WengFastComputationCertified2018}, and computing the Lipschitz constant 
is NP hard \cite{VirmauxLipschitzRegularityDeep2018}.

\section{Preliminaries} 
\label{sec:preliminaries}

\subsection{Notation} 
\label{sub:notation}
We will denote the real numbers by $\mathbb{R}$. For an $(n \times m)$ matrix 
(or vector), $A$, we will use the notation $\llbracket A \rrbracket_{[i,j]}$ to 
denote the element in the $i^\text{th}$ row and $j^\text{th}$ column of $A$. 
Analogously, the notation $\llbracket A \rrbracket_{[i,:]}$ will denote the 
$i^\text{th}$ row of $A$, and $\llbracket A \rrbracket_{[:, j]}$ will denote 
the $j^\text{th}$ column of $A$; when $A$ is a vector instead of a matrix, both 
notations will return a scalar corresponding to the corresponding element in 
the vector. We will use angle brackets $\;\subarg{ \cdot }$ to delineate the 
arguments to a function that \emph{returns a function}. We use one special form 
of this notation: for a function $f : \mathbb{R}^n \rightarrow \mathbb{R}^m$ 
and  $i\in \{1, \dots, m\}$ define $\pi_i \subarg{ f } : x \mapsto \llbracket 
f(x) \rrbracket_{[i,:]}$. Finally, $\lVert \cdot \rVert$ will refer to the  
max-norm on $\mathbb{R}^n$, unless otherwise specified.


\subsection{Neural Networks} 
\label{sub:neural_networks}
We consider only  Rectified Linear Unit Neural Networks (ReLU NNs). A $K$-layer 
ReLU NN is specified by $K$ \emph{layer} functions; a layer may be either 
linear or nonlinear. Both types of layer are specified by a parameter list 
$\theta \triangleq (W,b)$ where $W$ is a  $(\overline{d} \times \underline{d})$ 
matrix and $b$ is a $(\overline{d} \times 1)$ vector. Specifically, the 
\emph{linear} and \emph{nonlinear} layers specified by $\theta$ are denoted by 
$L_{\theta}$ and $L_{\theta}^{\scriptscriptstyle \sharp}$, respectively, and 
are defined as:
\begin{align}
	L_{\theta} &: \mathbb{R}^{\scriptscriptstyle \underline{d}} \rightarrow \mathbb{R}^{\scriptscriptstyle \overline{d}}, 
	   &L_{\theta} &:  z \mapsto Wz + b \\
	L_{\theta}^{\scriptscriptstyle \sharp} &: \mathbb{R}^{\scriptscriptstyle \underline{d}} \rightarrow \mathbb{R}^{\scriptscriptstyle \overline{d}},  
	    &L_{\theta}^{\scriptscriptstyle \sharp} &:  z \mapsto \max\{ L_{\theta}(z), 0 \}.
\end{align}
where the $\max$ function is taken element-wise. Thus, a $K$-layer ReLU NN 
function is specified by functionally composing $K$ such layer functions whose 
parameters $\theta^{|i}, i = 1, \dots, K$ have dimensions that satisfy 
$\underline{d}^{|i} = \overline{d}{\vphantom{d}}^{|i-1}: i = 2, \dots, K$; we  
will consistently use the superscript notation ${}^ {\scriptscriptstyle | k}$ 
to identify a parameter with layer $k$. Whether a layer function is linear or 
not will be further specified by a set of linear layers, $\mathtt{lin} 
\subseteq \{1, \dots, K\}$. For example, a typical $K$-layer NN has 
$\mathtt{lin} = \{K\}$, which together with a list of $K$ layer parameters 
defines the NN: $\nn = L_{\theta^{|K}}^{\vphantom{{\scriptscriptstyle \sharp}}} 
\circ L_{\theta^{|K-1}}^{\scriptscriptstyle \sharp} \circ \dots \circ 
L_{\theta^{|1}}^{\scriptscriptstyle \sharp}$.

To indicate the dependence on parameters, we will index a ReLU $\nn$ 
by a \emph{list of NN parameters} $\Theta \triangleq$ $( \mathtt{lin}, 
\theta^{|1},$ $\dots,$ $\theta^{|K} );$
i.e., we will often write $\nn\subarg{\Theta} :  \mathbb{R}^{\scriptscriptstyle 
\underline{d}\vphantom{d}^{|1}} \rightarrow  \mathbb{R}^{\scriptscriptstyle 
\overline{d}\vphantom{d}^{|K}}$.


\subsection{Two-Level-Lattice (TLL) Neural Networks} 
\label{sub:two_layer_lattice_neural_networks}
In this paper, we consider only Two-Level Lattice (TLL) ReLU NNs. Thus, we 
formally define NNs with the TLL architecture using the succinct method 
exhibited in \cite{FerlezBoundingComplexityFormally2020}; the material in this 
subsection is derived from \cite{FerlezAReNAssuredReLU2020, 
FerlezBoundingComplexityFormally2020}.

A TLL NN is most easily defined by way of three generic NN composition 
operators. Hence, the following three definitions lead to the TLL NN in 
{D}efinition \ref{def:scalar_tll}.
\begin{definition}[Sequential (Functional) Composition]
\label{def:functional_composition}
	Let $\nn\subarg{\Theta_{\scriptscriptstyle i}}:  
	\mathbb{R}^{\underline{d}^{\scriptscriptstyle |1}_i} \rightarrow 
	\mathbb{R}^{\overline{d}\vphantom{d}^{|K_i}_i}$, $i = 1, 2$ be two NNs with 
	parameter lists $\Theta_i \triangleq (\mathtt{lin}_i, 
	\theta^{\scriptscriptstyle |1}_i, \dots, \theta^{\scriptscriptstyle 
	|K_i}_i)$, $i=1,2$ such that  $\overline{d}\vphantom{d}^{\scriptscriptstyle 
	|K_1}_1 = \underline{d}^{\scriptscriptstyle |1}_2$. Then the 
	\textbf{sequential (or functional) composition} of 
	$\nn\subarg{\Theta_{\scriptscriptstyle 1}}$ and 
	$\nn\subarg{\Theta_{\scriptscriptstyle 2}}$, i.e.  
	$\nn\subarg{\Theta_{\scriptscriptstyle 1}} \circ  
	\nn\subarg{\Theta_{\scriptscriptstyle 2}}$, is a NN that is represented by 
	the parameter list $\Theta_{1} \circ \Theta_{2} \triangleq  (\mathtt{lin}_1 
	\cup (\mathtt{lin}_2 + K_1),  \theta^{\scriptscriptstyle |1}_1, \dots, 
	\theta^{\scriptscriptstyle |K_1}_{1}, \theta^{\scriptscriptstyle |1}_2,$ 
	$\dots, \theta^{\scriptscriptstyle |K_2}_2)$, where $\mathtt{lin}_2 
	\negthinspace + \negthinspace K_1$ is an element-wise sum.
\end{definition}
\begin{definition}
	\label{def:parallel_composition}
	Let $\nn\subarg{\Theta_{\scriptscriptstyle i}}:  
	\mathbb{R}^{\underline{d}^{\scriptscriptstyle |1}_i} \rightarrow 
	\mathbb{R}^{\overline{d}\vphantom{d}^{|K}_i}$, $i = 1, 2$ be two $K$-layer 
	NNs with parameter lists $\Theta_i = (\mathtt{lin},(W^{\scriptscriptstyle  
	|1}_i\negthinspace , b^{\scriptscriptstyle |1}_i), \dots,$ 
	$(W^{\scriptscriptstyle |K}_i \negthinspace, b^{\scriptscriptstyle 
	|K}_i))$, $i = 1,2$ such that $\underline{d}^{\scriptscriptstyle |1}_1 = 
	\underline{d}^{\scriptscriptstyle |1}_2$; also note the common set of  
	linear layers, $\mathtt{lin}$. Then the \textbf{parallel composition} of 
	$\nn\subarg{\Theta_{\scriptscriptstyle 1}}$ and  
	$\nn\subarg{\Theta_{\scriptscriptstyle 2}}$ is a NN given by: 
	\begin{equation}
		\Theta_{1} \negthickspace \parallel \negthickspace \Theta_{2} \negthinspace \triangleq \negthinspace \big(
			\mathtt{lin},
			\Big(
				\negthinspace
				\Big[
					\begin{smallmatrix}
						W^{\scriptscriptstyle |1}_1\\
						W^{\scriptscriptstyle |1}_2
					\end{smallmatrix}
				\Big],
				\Big[
					\begin{smallmatrix}
						b^{\scriptscriptstyle |1}_1 \\
						b^{\scriptscriptstyle |1}_2
					\end{smallmatrix}
				\Big]
				\negthinspace
			\Big),
			{\dots} ~, 
			\Big(
				\negthinspace
				\Big[
					\begin{smallmatrix}
						W^{\scriptscriptstyle |K}_1 & \mathbf{0} \\
						\mathbf{0} & W^{\scriptscriptstyle |K}_2
					\end{smallmatrix}
				\Big],
				\Big[
					\begin{smallmatrix}
						b^{\scriptscriptstyle |K}_1 \\
						b^{\scriptscriptstyle |K}_2
					\end{smallmatrix}
				\Big]
				\negthinspace
			\Big)
		\negthinspace\big)
	\end{equation}
	where $\mathbf{0}$ is a sub-matrix of zeros of the appropriate size. That 
	is $\Theta_{1} \negthickspace \parallel \negthickspace \Theta_{2}$ accepts 
	an input of the same size as (both) $\Theta_1$ and $\Theta_2$, but has as 
	many outputs as $\Theta_1$ and $\Theta_2$ combined.
\end{definition}

\begin{definition}[$n$-element $\min$/$\max$ NNs]
	\label{def:n-element_minmax_NN}
	An $n$\textbf{-element $\min$ network} is denoted by the parameter list 
	$\Theta_{\min_n}$. $\nn\subarg{\Theta_{\min_n}}: \mathbb{R}^n \rightarrow 
	\mathbb{R}$ such that $\nn\subarg{\Theta_{\min_n}}(x)$ is the minimum 
	from among the components of $x$ (i.e. minimum according to the usual order 
	relation $<$ on $\mathbb{R}$). An $n$\textbf{-element $\max$ network} is 
	denoted by $\Theta_{\max_n}$, and functions analogously. These networks are 
	described in \cite{FerlezAReNAssuredReLU2020}.
\end{definition}


The ReLU NNs defined in {D}efinition  
\ref{def:functional_composition}-\ref{def:n-element_minmax_NN} can be arranged 
to define a TLL NN as shown in \cite[Figure 1]{FerlezFastBATLLNN2022}.
\r{We formalize this  construction by first defining a \emph{scalar} TLL NN, 
and then extend this notion to a \emph{multi-output} TLL NN 
\cite{FerlezBoundingComplexityFormally2020}.}\marginnote{\r{(E4)}}
\begin{definition}[Scalar TLL NN {\cite{FerlezBoundingComplexityFormally2020}}]
\label{def:scalar_tll}
A NN from $\mathbb{R}^n \rightarrow \mathbb{R}$ is a \textbf{TLL NN of size} 
$(N,M)$ if its parameter list $\Xi_{\scriptscriptstyle N,M}$ can be 
characterized entirely by integers $N$ and $M$ as follows.
\begin{equation}
	\Xi_{N,M} \negthinspace \triangleq  \negthinspace
		\Theta_{\max_M} \negthinspace\negthinspace
	\circ \negthinspace
		\big(
			(\negthinspace\Theta_{\min_N} \negthinspace \circ \Theta_{S_1}\negthinspace) \negthinspace
			\parallel \negthinspace {\scriptstyle \dots} \negthinspace \parallel \negthinspace
			(\negthinspace\Theta_{\min_N} \negthinspace \circ \Theta_{S_M}\negthinspace)
		\big) \negthinspace
	\circ 
		\Theta_{\ell}
\end{equation}
where

\begin{itemize}
	\item $\Theta_\ell \triangleq (\{1\},\theta_\ell)$ for $\theta_\ell 
		\triangleq (W_\ell,b_\ell)$;

	\item  each $\Theta_{S_j}$ has the form $\Theta_{S_j} = \big( \{1\}, 
		\big( S_j, \mathbf{0} \big)\big)$ where $\mathbf{0}$ is the column  
		vector of $N$ zeros, and where

	\begin{itemize}
		\item[$\vcenter{\hbox{${\scriptscriptstyle \blacksquare}$}}$] $S_j = 
			\left[ \begin{smallmatrix} {\llbracket I_N \rrbracket_{[\iota_1, 
			:]}}\negthickspace\negthickspace\negthickspace^{^{\scriptscriptstyle\text{T}}} 
			& \; \dots \; & {\llbracket I_N \rrbracket_{[\iota_N, 
			:]}}\negthickspace\negthickspace\negthickspace^{^{\scriptscriptstyle\text{T}}} 
			\end{smallmatrix} \right]^\text{T}$ for a length-$N$ sequence 
			$\{\iota_k\}$ where $\iota_k \in \{1, \dots, N\}$ and $I_N$ is the 
			$(N \times N)$ identity matrix. 
		\end{itemize}
\end{itemize}

The affine functions implemented by the mapping  $\ell_i \triangleq  
\pi_i\subarg{L_{\theta_\ell}}$ for $i = 1, \dots, N$ will be referred to as the 
\textbf{local linear functions} of $\Xi_{N,M}$; we assume for simplicity that 
these affine functions are unique. The matrices $\{ S_j | j = 1, \dots, M\}$ 
will be referred to as the \textbf{selector matrices} of $\Xi_{N,M}$. Each set 
$s_j \triangleq \{ k \in \{1, \dots, N\} | \exists \iota \in \{1, \dots, N\}. 
\llbracket S_j \rrbracket_{[\iota,k]} = 1 \}$ is said to be the  
\textbf{selector set of} $S_j$.
\end{definition}

\begin{definition}[Multi-output TLL NN {\cite{FerlezBoundingComplexityFormally2020}}]
\label{def:multi-output_tll}
	A NN that maps $\mathbb{R}^n \rightarrow \mathbb{R}^m$ is said to be a 
	\textbf{multi-output TLL NN of size} $(N,M)$ if its parameter list 
	$\Xi_{\scriptscriptstyle N,M}^{(m)}$ can be written as 
	\begin{equation}
	\label{eq:multi_out_tll}
		\Xi_{\scriptscriptstyle N,M}^{(m)} = \Xi_{\scriptscriptstyle N, M}^1 
			\parallel
			\dots
			\parallel
			\Xi_{\scriptscriptstyle N, M}^m
	\end{equation}\\[-10pt]
	for $m$ equally-sized scalar TLL NNs, $\Xi_{\scriptscriptstyle N, M}^1, 
	\dots, \Xi_{\scriptscriptstyle N, M}^m$, which will be referred to as the 
	\textbf{(output) components of} $\Xi_{\scriptscriptstyle N,M}^{(m)}$.
\end{definition}

\marginnote{\g{(E5)}}\g{Finally, we have the following definition.
\begin{definition}[Non-degenerate TLL]
	\label{def:non_degenerate_tll}
	A scalar TLL NN $\Xi_{\scriptscriptstyle N,M}$ is \textbf{non-degenerate} 
	if each function $\ell_i \triangleq \pi_i\subarg{L_{\Theta_\ell}}$ (see  
	\cref{def:scalar_tll}) is realized on some open set. That is, for each $i = 
	1, \dots, N$ there exists an open set $V_i \subset \mathbb{R}^n$ such that
	\begin{equation}
		\forall x \in V_i . \nn\subarg{\Xi_{\scriptscriptstyle N,M}}(x) = \ell_i(x).
	\end{equation}
\end{definition}}

\section{Problem Formulation} 
\label{sec:problem_formulation}

The main object of our attention is the reachable set of a discrete-time LTI 
system in closed-loop with a state-feedback TLL NN controller. To this end, we  
define the following.

\begin{definition}[One-Step Closed-Loop Reachable Set]
\label{def:closed_loop_reachable_set}
Let $x_{t+1} = A x_t + B u_t$ be a discrete-time LTI system with states $x_t 
\in \mathbb{R}^n$ and controls $u_t \in \mathbb{R}^m$. Furthermore, let $X 
\subset \mathbb{R}^n$ be a compact, convex polytope, and let $\mu :  
\mathbb{R}^n \rightarrow \mathbb{R}^m$ be a state-feedback controller. Then the 
\textbf{one-step reachable set from $X_t$ under feedback control $\mu$} is 
defined as:
\begin{equation}
	X_{t+1} \triangleq \mathcal{R}(X_t, \mu) \triangleq \{ A x + B \mu(x) | x \in X_t\}.
\end{equation}

For a compact, convex polytope, $X_0 \subset \mathbb{R}^n$, the 
\textbf{$T$-step reachable set from $X_0$ under control $\mu$} is the set $X_T$ 
that is defined according to the recursion:
\begin{equation}
	X_t \triangleq \mathcal{R}(X_{t-1}, \mu), t = 1, \dots, T.
\end{equation}
\end{definition}

In one instance, we will be interested in computing $X_{t}$ \emph{exactly} from 
$X_{t-1}$ (or by recursive application, $X_0$). However, we will also be 
interested in two different approximations for the reachable set $X_{t}$:  a 
one-step bounding box for $X_{t}$ from $X_{t-1}$, and a bounding box for  
$X_{t}$ obtained by propagating bounding boxes from $X_0$. Thus, we have the 
following.

\begin{definition}[One-Step $\epsilon$-Bounding Box]
\label{def:one_step_box}
	Let $A$, $B$, $X_t$ and $\mu$ be as in 
	\cref{def:closed_loop_reachable_set}. Then a \textbf{one-step $\epsilon$  
	bounding box reachable from $X_t$} is a box $\mathtt{B}_{t+1} 
	\negthickspace = \negthinspace {\bigtimes}_{i=1}^{n} [l_i, r_i] 
	\negthinspace \subset \negthinspace \mathbb{R}^n$ s.t.:

	\begin{enumerate}[left=0.5\parindent,label={\itshape (\roman*)}]
		\item $X_{t+1} = \mathcal{R}(X_{t},\mu) \subseteq  
			\mathtt{B}_{t+1}$; and 

		\item for each $i = 1, \dots, n$, there exist points $x_{l_i},  
			x_{r_i} \in \mathcal{R}(X_{t},\mu)$ such that:
			\begin{equation}
			 	\left| \llbracket x_{l_i} \rrbracket_{[i,:]} - l_i\right| < \epsilon \text{ and } 
			 	\left| \llbracket x_{r_i} \rrbracket_{[i,:]} - r_i\right| < \epsilon.
			\end{equation}
	\end{enumerate}
\end{definition}

The idea of a one-step $\epsilon$ bounding box can be extended to  
\emph{approximate} reachability by propagating one-step bounding boxes 
recursively instead of the previous reachable set itself.

\begin{definition}[$\epsilon$-Bounding Box Propagation]
\label{def:bounding_box_propagation}
	Let $A$, $B$, $X_0$ and $\mu$ be as in \cref{def:closed_loop_reachable_set}.

	Let $\mathtt{B}_0 \triangleq X_0$ by convention. Then an 
	\textbf{$\epsilon$-bounding box propagation} of $X_0$ is a sequence of 
	bounding boxes, $\mathtt{B}_{t}^{X_0}$, $t = 0, \dots, T$ such that:

	\begin{itemize}
		\item for all $t = 1, \dots, T$, $\mathtt{B}_{t}^{X_0}$ is an  
			$\epsilon$-bounding box for the system with initial set of states 
			$\mathtt{B}_{t-1}^{X_0}$.
	\end{itemize}
\end{definition}
\noindent \marginnote{\g{(E6)}}\g{Note: although an $\epsilon$-bounding box 
propagation only approximates the reachable set $X_t$, the amount of 
over-approximation depends only on $\epsilon$ and the dynamics -- \emph{not the 
controller}. Thus, any desired approximation error to $X_t$ can be obtained by 
computing $\epsilon$-box propagations of suitably small \emph{subsets} of $X_0$ 
(to compensate for the propagation of each bounding box approximation through 
the dynamics).}

\e{Finally, as a consequence of considering ReLU NNs and 
polytopic state sets, our complexity results can be written in terms of the 
complexity of solving a linear program (LP).}\marginnote{\e{(E7)}}
\begin{definition}[LP Complexity]
\label{def:lp_complexity}
	Let $\text{LP}(\eta, \nu)$ be the complexity of an LP in dimension $\nu$ 
	with $\eta$ inequality constraints.
\end{definition}
This complexity is polynomial in both parameters, subject to the usual caveats 
associated with digital arithmetic.

\section{Exact Reachability for TLL NN Controllers} 
\label{sec:exact_reachability_for_tll_nn_controllers}
Our first complexity result shows that exact one-step reachability for an LTI 
system controlled by a TLL NN is computable in polynomial time in the size of 
the TLL.

\begin{theorem}\label{thm:one_step_exact}
	Let $A$, $B$, and $X_t$ be as defined in 
	\cref{def:closed_loop_reachable_set}, where $X_t$ is the intersection of 
	$N_{X_t}$ linear constraints. Moreover, suppose this system is controlled 
	by a state-feedback TLL NN controller $\nn\subarg{\Xi_{\scriptscriptstyle 
	N,M}^{(m)}} \negthickspace : \negthickspace \mathbb{R}^n \negthickspace 
	\rightarrow  \negthickspace \mathbb{R}^m$ 
	(\cref{sub:two_layer_lattice_neural_networks}).

	For a fixed state dimension, $n$, the reachable set $X_{t+1}$ can be 
	represented as the union of at most
		$O( m^n \cdot N^{2 n} / n! )$
	compact, convex polytopes, and these polytopes can be computed in time 
	complexity \g{at most} (also for fixed $n$):\marginnote{\g{(E8)}}
	\begin{equation*}
		O( (m \cdot N)^{2^{n+1}} \negthickspace \cdot m^{n+2} \cdot n \cdot M \cdot N^{2n + 3} \cdot \text{LP}(m N^2 + N_{X_t}, n) /n! ).
	\end{equation*}
\end{theorem}
\begin{proof}
	This follows almost directly from the result in 
	\cite{FerlezBoundingComplexityFormally2020}, where it is shown that a 
	multi-output TLL with parameters $\Xi_{\scriptscriptstyle N,M}^{(m)}$ has 
	at most as many linear (affine) regions as there are regions in a 
	hyperplane arrangement with $O(m \cdot N^2)$ hyperplanes.  Clearly, each of 
	these potential regions can contribute one polytope to the reachable set 
	$X_{t+1}$. According to \cite{FerlezBoundingComplexityFormally2020}, these 
	regions can be enumerated in time complexity:
	\begin{equation}\label{eq:main_thm_conclusion_proof}
		O(m^{n+2} \cdot n \cdot M \cdot N^{2n + 3} \cdot \text{LP}(m N^2 + N_{X_t}, n) /n!)
	\end{equation}
	which includes the complexity of identifying the active linear function on  
	each of those regions  \cite[Proposition  
	4]{FerlezBoundingComplexityFormally2020}. \marginnote{\g{(E9)}}\g{The LP 
	complexity in \eqref{eq:main_thm_conclusion_proof} depends on $N_{X_t}$ 
	because it is necessary to obtain the intersection of the $O(m \cdot N^2)$ 
	regions with $X_{t}$.}

	Thus, it remains to determine the reachable set with respect to the TLL's 
	realized affine function on each such region. The complexity of this 
	operation is bounded by the complexity of transforming each such polytope 
	through the $A$ matrix and $B$ times the affine function realized by the 
	TLL on that region. This can be accomplished by Fourier-Motzkin elimination 
	to determine the resulting  polytopes that add together to form the 
	associated constituent polytope  of $X_{t+1}$. This operation has 
	complexity $O((m \cdot N)^{2^{n+1}})$.
\end{proof}






\section{Bounding-Box Reachability for TLL NN Controllers} 
\label{sec:bounding_box_reachability_for_tll_nn_controllers}

\begin{figure}[t]
\vspace{6pt}
\centering
\includegraphics[width=.83\linewidth,trim={0cm 0cm 0cm 0cm},clip]{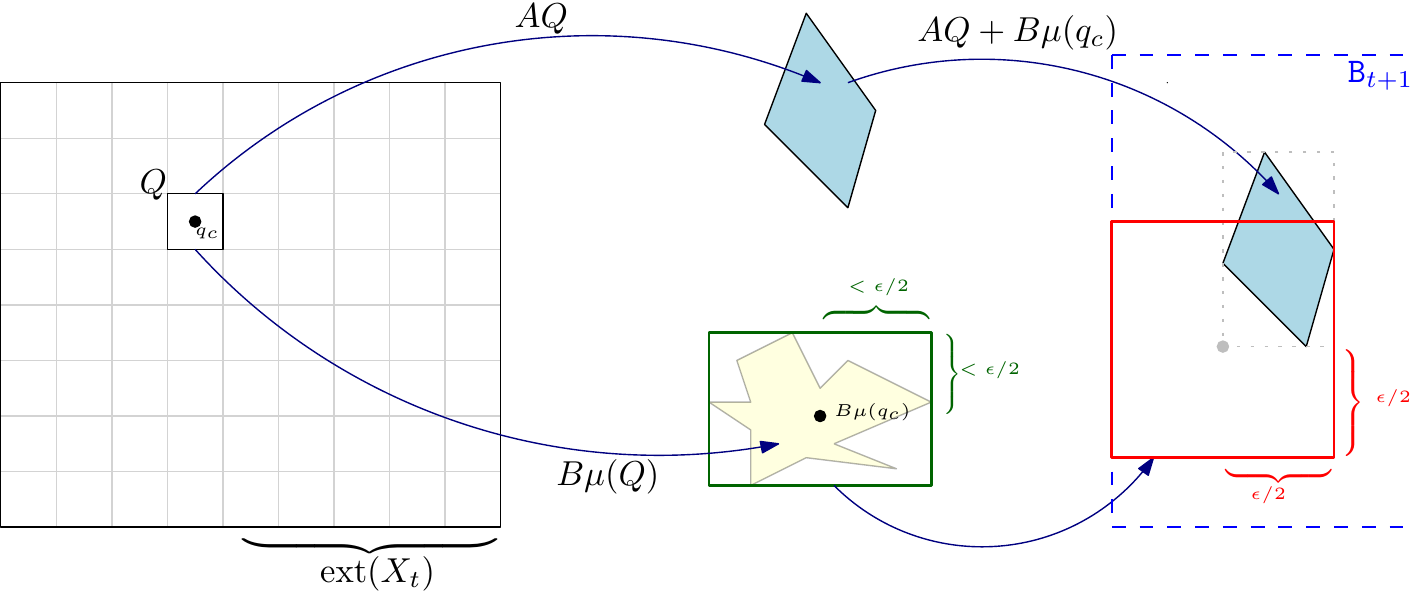}%
\caption{Illustration of the proof of \cref{prop:generic_lipschitz_reachability}}
\label{fig:lipschitz_reachability}
\vspace{-20pt}
\end{figure}

We begin with the following useful definition.

\begin{definition}[Center/Extent of $X \subset \mathbb{R}^n$]
	Let $X \subset \mathbb{R}^n$ be a compact set. Then the \textbf{center of 
	$X$} is the point $x_c \in \mathbb{R}^n$ such that for each $i = 1, \dots, 
	n$:
	\begin{equation}
		\llbracket x_c \rrbracket_{[i,:]} \triangleq \tfrac{1}{2} \cdot 
			\big( 
				\min_{\scriptscriptstyle x\in X} \llbracket x \rrbracket_{[i,:]}
				+
				\max_{\scriptscriptstyle x\in X} \llbracket x \rrbracket_{[i,:]}
			\big)
	\end{equation}

	Also, define the \textbf{extent of $X$ along coordinate $i$} as:
	\begin{equation}
		\text{ext}_i(X) \triangleq \max_{x \in X} \llbracket x - x_c \rrbracket_{[i,:]},
	\end{equation}
	and the \textbf{extent of $X$} as $\text{ext}(X) \triangleq  \max_{i = 1, 
	\dots, n} \text{ext}_i(X)$.
\end{definition}

\subsection{One-Step $\epsilon$ Bounding Box Reachability} 
\label{sub:one_step_bounding_box_reachability}

Now we can state our second main result: that a one-step bounding box can be 
computed in polynomial time in the size of a TLL NN controller. We provide two  
such results, each of which is polynomial in different aspects of the problem.

\begin{theorem}\label{thm:tll_size_one_step_box}
	Let $A$, $B$, $X_t$ and $\nn\subarg{\Xi_{\scriptscriptstyle N,M}^{(m)}}$ be 
	as in the statement of \cref{thm:one_step_exact}.

	Then a one-step $\epsilon=0$ bounding box from $X_t$ is computable in time 
	\marginnote{\g{(E10)}}complexity \g{at most} (for fixed dimension, $n$):
	\begin{equation}\label{eq:box_exact_complexity}
		O( m^{n+2} \cdot n^2 \cdot M \cdot N^{2n + 3} \cdot \text{LP}(m N^2 + N_{X_t}, n) /n! ).
	\end{equation}
\end{theorem}
\begin{proof}
This follows almost directly from the proof of \cref{thm:one_step_exact}. For 
each of the convex polytopes describing the reachable set, the Fourier-Motzkin 
elimination can be replaced by $2 \cdot n$ LPs to compute its bounding box; 
these can be combined to determine an $\epsilon=0$ bounding box for $X_{t+1}$  
without increasing the complexity noted above.
\end{proof}

The result in \cref{thm:tll_size_one_step_box} certainly meets the criteria of 
a polynomial-time computation of a one-step bounding box. Unfortunately, the 
dependence on $N$ and $M$ in \cref{thm:tll_size_one_step_box} is significant 
despite being polynomial.  However, since we are considering only bounding box 
reachability, it makes sense to regard the TLL controller as a generic 
Lipschitz-continuous controller instead: this allows the dependence on its size 
to be replaced with a dependence on its Lipschitz constant, at the expense of 
additional (polynomial) dependence on the size of the set $X_t$ and the norm of 
the matrix $B$.

\begin{proposition}\label{prop:generic_lipschitz_reachability}
	Let $A$, $B$, $X_t$ and $\mu$ be as in the statement of 
	\cref{def:closed_loop_reachable_set}. Furthermore, suppose that $\mu$ is 
	Lipschitz continuous on $X_t$, with Lipschitz constant at most $\lVert \mu 
	\rVert$.

	\marginnote{\g{(E11)}}Then an $\epsilon$-bounding box from $X_t$ is 
	computable in complexity \g{at  most} (for fixed dimension, $n$):
	\begin{equation}\label{eq:lipschitz_reachability_generic}
		O( (2\cdot\text{ext}(X_t) \cdot \tfrac{2 \lVert B \rVert \cdot \lVert \mu \rVert}{\epsilon} )^n 
		\cdot LP(2\cdot n, n)).
	\end{equation}
\end{proposition}
\begin{proof}
$X_t$ can be covered by a grid of $(2\cdot\text{ext}(X_t) \cdot 2 \lVert B 
\rVert \cdot \lVert \mu \rVert/\epsilon)^n$ hypercubes whose edges are of width 
$\epsilon/(2\lVert B \rVert \cdot \lVert \mu \rVert)$. Denote by $Q$ an 
arbitrary such hypercube, and let $q_c$ denote its center. Now observe that for 
all $q \in Q$:
\begin{equation}\label{eq:generic_lipschitz_output_box}
	\lVert B \mu(q) - B \mu(q_c) \lVert \leq  
\lVert B \rVert \cdot \lVert \mu \rVert \cdot \lVert q - q_c \rVert \leq  
\tfrac{\epsilon}{2} < \epsilon.
\end{equation}
Consequently, the set $Q^\prime \triangleq \{x : \lVert A Q + B \mu(q_c) - x 
\rVert < \epsilon\}$ is guaranteed to be in any $\epsilon$-bounding box of 
$X_{t+1}$, as is the exact bounding thereof, which we denote by  
$\text{box}(Q^\prime)$; see \cref{fig:lipschitz_reachability}. Clearly 
$\cup_{Q}  \text{box}(Q^\prime)$ is likewise so contained.

Thus, an $\epsilon$-bounding box from $X_t$ can be obtained by examining each  
$Q$ and computing $\text{box}(Q^\prime)$. The latter operation entails 
computing a bounding box for $AQ$, which has the complexity of $LP(2 \cdot n, 
n)$; see \cref{fig:lipschitz_reachability}.
\end{proof}

In particular, for some problems, the quantities in 
\eqref{eq:lipschitz_reachability_generic} may be much smaller than terms like 
$N^{2n}$ in \eqref{eq:box_exact_complexity}. In fact, this explains why this 
type of result is typically not used for NN reachability: it is computationally 
expensive to compute the exact Lipschitz constant of a generic NN -- indeed, it 
is of exponential complexity in the number of neurons for a deep NN. For a 
non-degenerate TLL NN, however, it is trivial to compute its exact Lipschitz 
constant (over all of $\mathbb{R}^n$)\footnote{Even when this bound is 
approximate, it depends only on the parameters of one layer; for general deep 
ReLU NNs, a bound of similar computational complexity involves multiplying 
weight matrices of successive layers.}.

\begin{lemma}\label{lem:lipschitz_bound}
	A bound on the Lipschitz constant of a TLL NN $\Xi_{\scriptscriptstyle 
	N,M}^{(m)}$ over $\mathbb{R}^n$ is computable in complexity $O(m 
	\negthinspace \cdot \negthinspace N \negthinspace \cdot \negthinspace n)$. 
	For a non-degenerate TLL (\cref{def:non_degenerate_tll}), this bound is 
	tight.
\end{lemma}
\begin{proof}
	This is a straightforward application of \cite[Proposition  
	3]{FerlezBoundingComplexityFormally2020} or the related result  
	\cite[Proposition 4]{CruzSafebyRepairConvexOptimization2021}. The only 
	affine functions realizable by a TLL are those described by its linear 
	layer (see \cref{def:scalar_tll}). If the TLL is not degenerate, then each 
	of those is realized in its output, hence also lower bounding its Lipschitz 
	constant. The claim follows, since there are $N$ such local linear 
	functions per output, each of whose Lipschitz constants can be computed in 
	$O(n)$.
\end{proof}

\begin{theorem}\label{thm:box_alg_choice}
	Let $A$, $B$, $X_t$ and $\nn\subarg{\Xi_{\scriptscriptstyle N,M}^{(m)}}$ be 
	as in the statement of \cref{thm:one_step_exact}. 

	Then for any $\epsilon > 0$, an  $\epsilon$-bounding box from $X_t$ can  be 
	computed with complexity no more than the maximum of 
	\eqref{eq:box_exact_complexity} and 
	\eqref{eq:lipschitz_reachability_generic} for $\lVert 
	\nn\subarg{\Xi_{\scriptscriptstyle N,M}^{(m)}} \rVert$ bounded according to 
	\cref{lem:lipschitz_bound}.
\end{theorem}






\section{The \myalg ~Algorithm} 
\label{sec:algorithm}

\marginnote{\e{(E12)}}\e{\cref{thm:box_alg_choice} establishes a trade-off in 
computational complexity between two methods for computing an 
$\epsilon$-bounding box from states $X_t$. However, it requires a commitment to 
the full computational complexity of one algorithm or the other. Moreover, the 
difference in computational complexity between \cref{thm:tll_size_one_step_box} 
and \cref{prop:generic_lipschitz_reachability} depends on the characteristics 
of the TLL controller on the set $X_t$ (assuming $\text{ext}(X_t)$ and  $\lVert 
B \rVert$ are fixed). If the TLL controller has relatively \emph{few linear 
regions} intersecting $X_t$ but a relatively \emph{large Lipschitz constant} on 
$X_t$, then \cref{thm:tll_size_one_step_box} will have lower complexity; recall 
that \cref{thm:tll_size_one_step_box} enumerates the linear regions of the TLL 
controller that intersect $X_t$. If, on the other hand, the TLL controller has 
relatively \emph{numerous linear regions} intersecting $X_t$ but a relatively 
\emph{small Lipschitz constant} on $X_t$, then 
\cref{prop:generic_lipschitz_reachability} will instead have lower complexity.  
Thus, the trade off in complexity between \cref{thm:tll_size_one_step_box} and 
\cref{prop:generic_lipschitz_reachability} amounts to roughly the following: 
reachability via \cref{prop:generic_lipschitz_reachability} is more efficient 
when the Lipschitz constant of the TLL  controller yields a partition of $X_t$ 
into hypercubes (see proof of \cref{prop:generic_lipschitz_reachability}) each 
of which ``typically'' contains many linear regions of the TLL.

This is a salient observation in light of the way that 
\cref{prop:generic_lipschitz_reachability} employs the Lipschitz constant of 
the TLL controller in question. In \cref{prop:generic_lipschitz_reachability}, 
the Lipschitz constant of the TLL controller is really used to create  a 
subdivision of $X_t$ into sets such that the output of the TLL controller lies  
within box of sufficiently small width: see 
\eqref{eq:generic_lipschitz_output_box}. However, note that the Lipschitz 
constant of a TLL controller over the \emph{entire set $X_t$} will generally be 
larger than the Lipschitz constant of the TLL controller on any subset of 
$X_t$. Also, there may be subsets of $X_t$ where the TLL controller rapidly 
switches between large-Lipschitz-constant affine functions such that its output 
is nevertheless confined to a small box (e.g. a high-frequency saw-tooth 
function with small amplitude). This suggests the following improvement on 
\cref{thm:tll_size_one_step_box} and 
\cref{prop:generic_lipschitz_reachability}: identify large subsets of $X_t$ 
where the output of the TLL controller is bounded within a small box -- thereby 
replacing the enumeration of many linear regions of the TLL controller, as in 
\cref{thm:tll_size_one_step_box}, \emph{or} the enumeration of many 
``Lipschitz-width'' hypercubes, as in 
\cref{prop:generic_lipschitz_reachability}.

Thus, we introduce \myalg~as a practical, ``adaptive'' algorithm, which 
implements this strategy via the recent tool FastBATLLNN 
\cite{FerlezFastBATLLNN2022}. In particular, FastBATLLNN provides a fast 
algorithm for obtaining an $\epsilon$-tight bounding box on the output of a TLL 
controller subject to a convex, polytopic input constraint. This means that 
FastBATLLNN can be used to directly and efficiently identify subsets of $X_t$ 
where the output of the TLL controller is confined to a small box, without 
enumerating all of  the linear regions of the TLL. That is, for a convex, 
polytopic set $P  \subset X_t$ and a box $\mathsf{B} =  \bigtimes_{i  = 1}^m 
[l_i, r_i]$ FastBATLLNN can efficiently decide the query:
\begin{equation}\label{eq:fastbatllnn_query}
	\forall x \in P . \nn\subarg{\Xi_{\scriptscriptstyle N,M}^{(m)}}(x) \in \mathsf{B}
\end{equation}
by an algorithm that has \emph{half} the crucial exponent of the algorithms of 
\cref{thm:one_step_exact} and \cref{thm:tll_size_one_step_box} 
\cite{FerlezBoundingComplexityFormally2020} -- i.e., \textbf{without 
enumerating the affine regions of the TLL}. Thus, for $P$ as above, a tight 
bounding box on $\nn\subarg{\Xi_{\scriptscriptstyle  N,M}^{(m)}}(P)$ can be 
obtained by roughly $\log$ invocations of FastBATLLNN in a binary search on the 
endpoints of a bounding box.

In this context, the structure of \myalg~can be summarized as follows. Start 
with a hypercube  $Q_1^0$ of edge length $2\cdot \text{ext}(X_t)$, so that 
$Q_1^0$ is large enough to capture the whole set $X_t$. Then use FastBATLLNN to 
determine a sufficiently tight bounding box on the set 
$\nn\subarg{\Xi_{\scriptscriptstyle  N,M}^{(m)}}(Q_1^0 \cap X_t)$. If this 
bounding box is small enough that its endpoints differ by less than 
$\epsilon/\lVert  B\rVert$ from its center (see 
\eqref{eq:generic_lipschitz_output_box}), then the reachable bounding box, 
$\mathsf{B}_{t+1}$ can be updated directly as in 
\cref{prop:generic_lipschitz_reachability}. Otherwise, $Q_1^0$ should be 
refined into $2^n$ hypercubes, each with half its edge lengths, denoted by 
$Q_p^1$, $p=1, \dots, 2^n$, and the process is repeated recursively on each. As 
above, the recursion stops for a hypercube $Q_p^d$ at depth $d$ only if 
FastBATLLNN returns an bounding box for $\nn\subarg{\Xi_{\scriptscriptstyle  
N,M}^{(m)}}(Q_p^d \cap X_t)$ whose endpoints are within $\epsilon/\lVert 
B\rVert$.}


\marginnote{\g{(E13)}}\g{This basic recursion is described in 
\cref{alg:ltllbox}. Three functions in \cref{alg:ltllbox} require explanation:
\begin{itemize}
	\item $\mathtt{BBox}(P)$ computes an exact bounding box for the convex 
		polytopic set $P$ using LPs;

	\item $\mathtt{Subdivide}(Q,p)$ returns the $p^\text{th}$ hypercube 
		obtained by splitting each edge of the hypercube $Q$ in half;

	\item $\mathtt{FastBATLLNN}(\Xi_{\scriptscriptstyle N,M}^{(m)}, P,  
		\epsilon)$ returns an $\epsilon$-tight bounding box on the set 
		$\nn\subarg{\Xi_{\scriptscriptstyle N,M}^{(m)}} (P)$. 
\end{itemize}
}
\setlength{\textfloatsep}{0pt}
\IncMargin{0.5em}
%
%
\begin{figure}
\vspace{4pt}
	\begin{minipage}[t]{\linewidth}
\begingroup %
\removeonelatexerror
\begin{algorithm}[H]


\SetKwData{false}{False}
\SetKwData{true}{True}
\SetKwData{feas}{Feasible}

\SetKwData{constraints}{constraints}
\SetKwData{status}{status}
\SetKwData{sol}{sol}
\SetKwData{sollist}{solnList}
\SetKwData{actconstr}{actConstr}

\SetKwData{ha}{h\_a}
\SetKwData{reg}{reg}
\SetKwData{neghypers}{negHypers}
\SetKwData{i}{i}
\SetKwData{j}{j}
\SetKwData{p}{p}
\SetKwData{tllbbox}{TLLBx}

\SetKwFunction{append}{append}
\SetKwFunction{verifyScalarLB}{verifyScalarLB}
\SetKwFunction{ltllbox}{LTLLBox}
\SetKwFunction{solvefeas}{SolveLinFeas}
\SetKwFunction{findint}{FindInt}
\SetKwFunction{all}{all}
\SetKwFunction{regions}{Regions}
\SetKwFunction{L}{l}
\SetKwFunction{getneghypers}{NegativeHyperplanes}
\SetKwFunction{bbox}{BBox}
\SetKwFunction{subdivide}{Subdivide}
\SetKwFunction{fastbatllnn}{FastBATLLNN}
\SetKwFunction{wid}{Width}
\SetKwFunction{mn}{Min}
\SetKwFunction{mx}{Max}

\SetKw{Break}{break}
\SetKw{NOT}{not}
\SetKw{foriter}{for}
\SetKw{OR}{or} 
\SetKw{IN}{in}
\SetKw{CONT}{continue}
\SetKw{GLOBAL}{global}

\SetKwInOut{Input}{input}
\SetKwInOut{Output}{output}

\Input{
\hspace{3pt}$\epsilon > 0$ \\
\hspace{3pt}$A$ an $(n \times n)$ matrix \\
\hspace{3pt}$B$ an $(n \times m)$  matrix \\
\hspace{3pt}$X_t$ a compact, convex polytope of states \\
\hspace{3pt}$\Xi_{\scriptscriptstyle N,M}^{(m)}$, parameters of a TLL NN to verify
}
\Output{
\hspace{3pt}$\mathtt{B}_{t+1} = \bigtimes_{i=1}^n [l_i, r_i]$, an $\epsilon$-bounding box from $X_t$
}
\marginnote{\g{(E14)}}
\BlankLine

\GLOBAL $d \leftarrow 0$

\GLOBAL $\mathtt{B}_{t+1} \triangleq \bigtimes_{i=1}^n [l_i, r_i] \leftarrow \bigtimes_{i=1}^n [\infty, -\infty]$

\BlankLine
\SetKwProg{Fn}{function}{}{end}%
\Fn{\ltllbox{$\epsilon$, $A$, $B$, $X_t$, $\Xi_{\scriptscriptstyle N,M}^{(m)}$}}{

	$d \leftarrow d + 1$ \tcp*[h]{\g{Increment depth counter}}

	$Q_\text{local}^d \leftarrow $ \bbox{$X_t$}

	\BlankLine

	\For{\p \IN $1, \dots, 2^n$}{

		\tcp*[h]{\g{Subdivide $Q_\text{local}^d$ into $2^n$ hypercubes}}

		$Q_p^{d+1} \leftarrow $ \subdivide{$Q_\text{local}^d$, \p}

		\BlankLine

		\tcc{Get a bounding box on TLL output to $\epsilon/2$ error for inputs in $Q_p^{d+1} \cap X_t$}

		\tllbbox $\leftarrow$ \fastbatllnn{$\Xi_{\scriptscriptstyle N,M}^{(m)}$, $Q_p^{d+1}\cap X_t$, $\tfrac{\epsilon}{2}$}

		\BlankLine

		\uIf{\wid{\tllbbox} $< \epsilon/(2 \cdot \lVert B \rVert)$}{

			\tcc{Output of TLL is small enough on $Q_p^{d+1}$ that we can update $\mathtt{B}_{t+1}$}

			\For{\i \IN $1,\dots, n$}{ 
				\If{\mn{$\llbracket$\bbox{$A Q_p^{d+1}$}$\rrbracket_{[i,:]}$} - \mn{$\llbracket$\tllbbox$\rrbracket_{[i,:]}$} $< l_i$ }{

					$l_i \leftarrow $ \mn{$\llbracket$\tllbbox$\rrbracket_{[i,:]}$}

				}
				\If{\mx{$\llbracket$\bbox{$A Q_p^{d+1}$}$\rrbracket_{[i,:]}$} + \mx{$\llbracket$\tllbbox$\rrbracket_{[i,:]}$} $> r_i$ }{

					$r_i \leftarrow $ \mx{$\llbracket$\tllbbox$\rrbracket_{[i,:]}$}

				}
			}

			\Return
		}
		\Else(\tcp*[h]{Need to refine on $Q_p^{d+1}$}){ 

			\ltllbox{$\epsilon$, $A$, $B$, $Q_p^{d+1} \cap X_t$, $\Xi_{\scriptscriptstyle N,M}^{(m)}$}

		}

	}

}
\caption{\texttt{L-TLL Box} core recursion}
\label{alg:ltllbox}
\end{algorithm}
\endgroup
	\end{minipage}
\end{figure}
\DecMargin{0.5em}
%

Formally, we have the following Theorem, which describes the  \emph{worst-case} 
runtime of \myalg.

\begin{theorem}
	Let $A$, $B$, $X_t$ and $\nn\subarg{\Xi_{\scriptscriptstyle N,M}^{(m)}}$ be 
	as in the statement of \cref{thm:box_alg_choice}.

	Then for any $\epsilon > 0$, \myalg~can compute an $\epsilon$ bounding  box 
	from $X_t$ with a worst-case time complexity of 
	\begin{multline}\label{eq:myalg_run_time}
		O\Big(
		\log_2\big[ \epsilon \cdot
			\lVert \nn\subarg{\Xi_{\scriptscriptstyle N,M}^{(m)}} \rVert \cdot \text{ext}(X_t)
				\cdot
				\log_2\big[
					\tfrac{\text{ext}(X_t)\cdot \lVert B \rVert 
					\cdot \lVert \mu \rVert}{\epsilon}
				\big] 
		\big]
		\\
		m \cdot K \cdot 2^{K\cdot n} \cdot n \cdot M \cdot \tfrac{N^{n+3}}{n!} \cdot \text{LP}(N + N_{X_t}, n) \Big)
	\end{multline}
	where 
		$K = \big\lceil \log_2\big( 2 \cdot \text{ext}(X_t) \cdot \tfrac{2 \lVert B \rVert \cdot \lVert \mu \rVert}{\epsilon}  \big) \big\rceil$.
\end{theorem}
\begin{proof}
	\myalg~recursively subdivides a hypercube of edge length at most $2\cdot 
	\text{ext}(X_t)$ into $2^n$ hypercubes with each recursion. Let $d$ denote 
	the number of recursions, so that at depth $d$, \myalg has created at most 
	$2^{d \cdot n}$ hypercubes. Let $Q_p^d$, $p = 1, \dots, 2^{d \cdot n}$ 
	denote the hypercubes at depth $d$.

	In the worst case, \myalg~must recurse on every single hypercube at each 
	depth until \emph{all} of the resultant subdivided hypercubes have edge 
	length at most $2 \cdot \text{ext}(X_t) \cdot (2 \lVert B \rVert \cdot 
	\lVert \mu \rVert)/\epsilon$. This explains the factor $K \cdot 2^{K \cdot 
	n}$: the complexity at each depth is \emph{added} to the runtime, so the 
	cumulative runtime is dominated by the runtime for the largest recursion 
	depth.

	On any given subdivided hypercube, $Q_p^d$, the complexity of \myalg~is 
	dominated by using FastBATLLNN in a binary search on each of the $m$ 
	real-valued TLLs comprising $\nn\subarg{\Xi_{\scriptscriptstyle  
	N,M}^{(m)}}$. This is necessary to determine an $\epsilon/2$ bounding box 
	on the output of $\nn\subarg{\Xi_{\scriptscriptstyle N,M}^{(m)}}$ when its 
	input is constrained to the set $Q_p^d$. Since the Lipschitz constant of 
	$\nn\subarg{\Xi_{\scriptscriptstyle N,M}^{(m)}}$ is known, the invocations 
	of FastBATLLNN on each output is associated with a binary search over an 
	interval of width
	$2 \cdot \lVert \nn\subarg{\Xi_{\scriptscriptstyle N,M}^{(m)}} \rVert \cdot 2 \cdot \text{ext}(X_t)/2^d$
	until iterations of the search are lie in an interval of width 
	$\epsilon/2$. Thus, each output requires 
	\vspace{-9pt}
	\begin{equation}
		\big\lceil \log_2\big(\epsilon \cdot \lVert \nn\subarg{\Xi_{\scriptscriptstyle N,M}^{(m)}} \rVert \cdot \text{ext}(X_t) / 2^{3-d} \big) \big\rceil
	\end{equation}
	invocations of FastBATLLNN, so the total number of invocations of 
	FastBATLLNN is less than:
	\begin{equation*}
		O\big(
		\log_2\big[
			\epsilon \cdot 
			\lVert \nn\subarg{\Xi_{\scriptscriptstyle N,M}^{(m)}} \rVert \cdot \text{ext}(X_t)
				\cdot
				\log_2\big[
					\tfrac{\text{ext}(X_t)\cdot \lVert B \rVert 
					\cdot \lVert \mu \rVert}{\epsilon}
				\big] 
		\big]
		\big)
	\end{equation*}
	From \cite{FerlezFastBATLLNN2022}, each invocation of FastBATLLNN has  
	complexity:
	\begin{equation}
		n \cdot M \cdot N^{n+3} \cdot \text{LP}(N + N_{X_t}, n) / n!.
	\end{equation}
	This explains the formula \eqref{eq:myalg_run_time}.
\end{proof}



\vspace{-15pt}
\section{Experiments} 
\label{sec:experiments}
\vspace{-5pt}
To evaluate \myalg~as an LTI reachability tool, we used it to perform 
multi-step bounding box propagation on a number of TLL NN controllers; see 
\cref{def:bounding_box_propagation}. We compared the results to NNV's 
\cite{TranNNVNeuralNetwork2020} approximate reachability analysis setting. For 
this evaluation we selected 40 networks from the TLL Verification Benchmark in 
the 2022 VNN competition \cite{vnn2022}: the first 10 examples from each of the 
sizes $N=M=8, 16, 24$ and  $32$ were used, and these TLLs were converted to a 
fully-connected Tensorflow format that NNV could import. Each TLL had $n=2$ 
inputs and $m = 1$ output, so we took these as our state and control 
dimensions, respectively and generated one random $A$ and $B$ matrix for each 
TLL NN. We likewise generated one polytopic set of states to serve as $X_0$ for 
each TLL/system combination. Reachability analysis was performed on both tools 
for $T = 3$ discrete time steps. Both tools were given at most 4 days of 
compute time per TLL/system combination on a standard Microsoft Azure E2ds v5 
instance; this instance has one CPU core running at 2.8GHz and  16Gb of  RAM  
and 32Gb swap.

The execution time results of this experiment are summarized by the 
box-and-whisker plot in \cref{fig:exec_time}. \myalg~was able to complete all 
reachability problems  well within the timeout. However, NNV only completed the 
10 reachability problems for size $N = 8$ and one reachability problem for size 
$N = 16$; it timed out at 4 days for all other problems. On problems where both 
tools completed the entire reachability analysis, \myalg~ranged from 32x faster 
(the first instance of size $N = 8$) to 5,392x faster (the commonly completed 
instance of $N = 16$). For problems that  both algorithms finished, the final 
reachability boxes produced by \myalg were anywhere from 0.08 to 1.42 times the 
area of those produced by NNV. \cref{fig:example_boxes} shows one sequence of 
reach boxes output by \myalg and NNV.

\vspace{-6pt} \setlength\floatsep {2pt}
\begin{figure}
\vspace{8pt}
\centering%
\includegraphics[width=.89\linewidth,trim={1.6cm 0cm 0cm 0.1cm},clip]{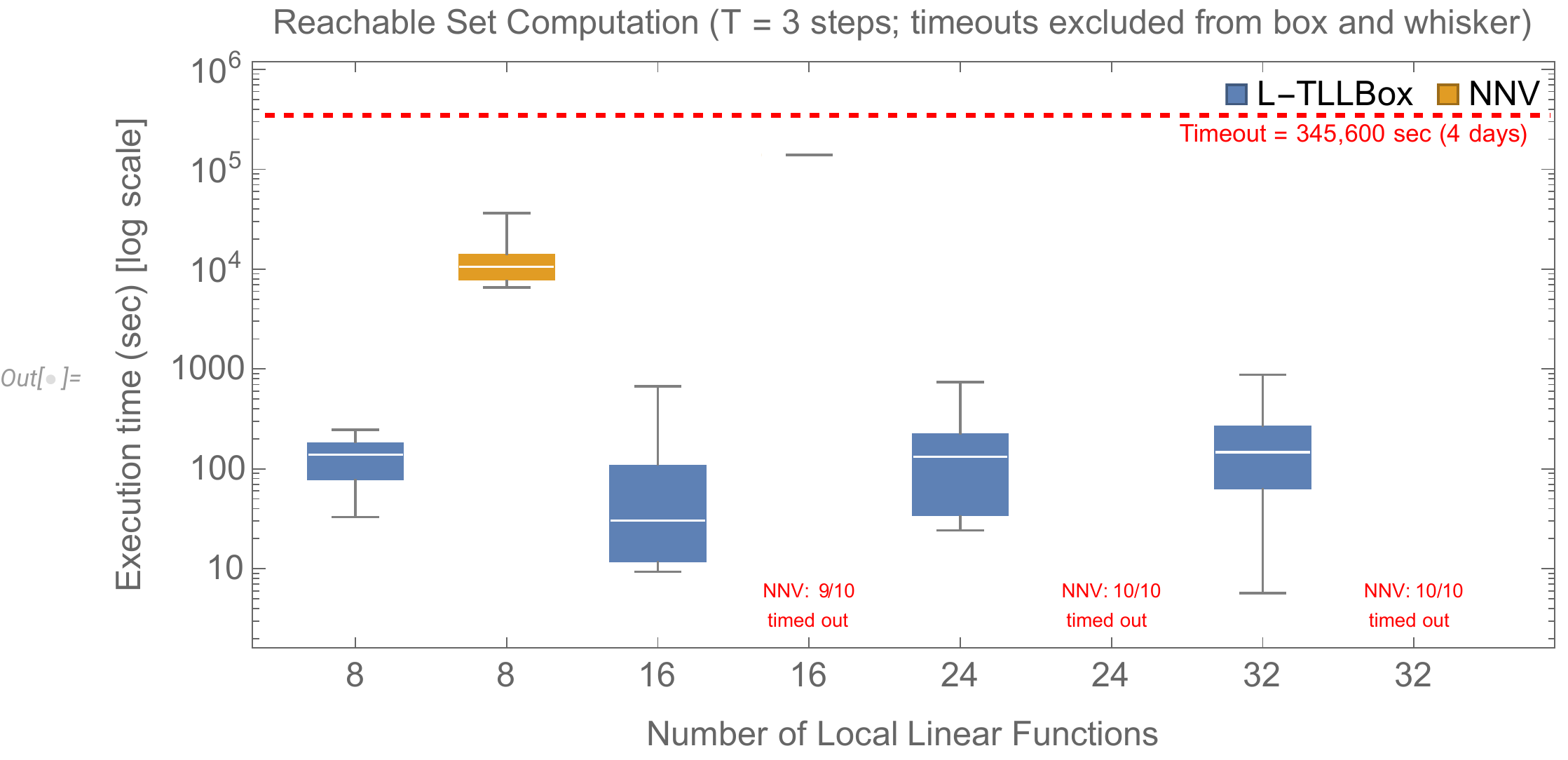}%
\vspace{-3pt}%
\caption{Execution time of \myalg~compared to NNV}
\label{fig:exec_time}
\end{figure}

\begin{figure}
\vspace{8pt}
\centering%
\includegraphics[width=.93\linewidth,trim={4.4cm 0.5cm 0cm 0.5cm},clip]{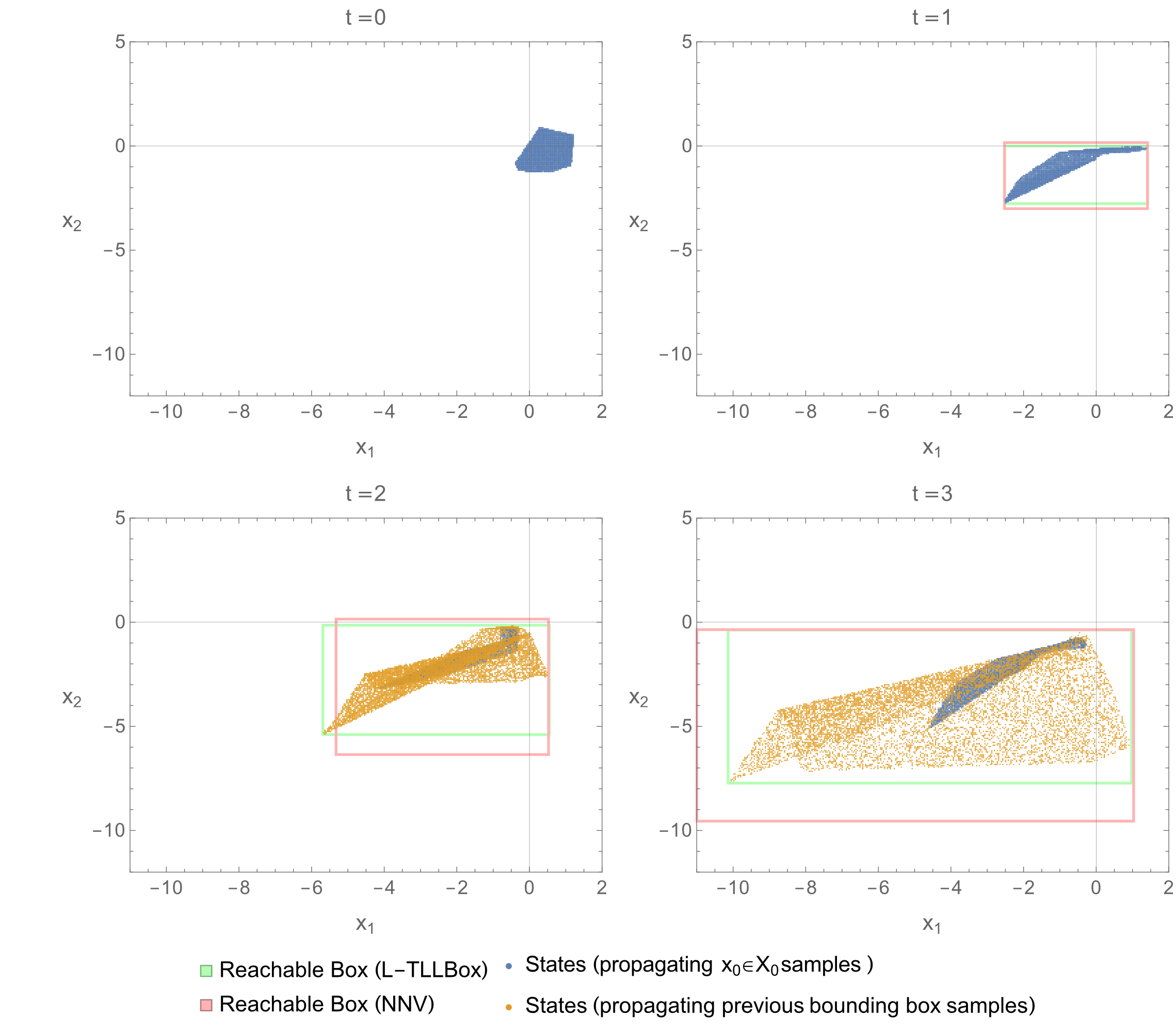}%
\vspace{-3pt}%
\caption{Example reachability box progressions computed by \myalg~(25 sec.) and NNV (139,000 sec.); this was the sole $N=16$ reachability sequence completed by NNV.}
\label{fig:example_boxes}
\vspace{-5pt}
\end{figure}


\g{%
\section{Conclusions} 
\label{sec:conclusions}
\marginnote{\g{(E15)}}In this paper, we presented several polynomial complexity 
results for reachability of an LTI system with a TLL NN controller, including 
\myalg; these results improve on the exponential complexity for the same 
reachability problem with a general NN controller, and thus provide a 
motivation for designing NN controllers using the TLL architecture. As a 
result, there are a numerous opportunities for future work such as: considering 
reachability for more general convex sets (e.g. ellipsoidal sets), and 
generalizing FastBATLLNN so that \myalg~can be extended to exact reachability. }%








\bibliographystyle{plain} %
\bibliography{mybib}

\end{document}